\documentclass[conference]{IEEEtran}
\usepackage{amsmath,amssymb,amsfonts,mathtools,amsthm}
\usepackage{graphicx}
\usepackage{textcomp}
\usepackage{xcolor}
\usepackage{booktabs}
\usepackage{algorithm}
\usepackage{algpseudocode}

\newcommand{\spr}{s^{\prime}}

\newcommand{\argmax}{\operatornamewithlimits{argmax}}

\usepackage[english]{babel}
\usepackage[square,numbers]{natbib}
\setcitestyle{authoryear} 

\newtheorem{theorem}{Theorem}[section]

\newtheorem{lemma}[theorem]{Lemma}

\begin{document}

\title{Strategy Masking: A Method for Guardrails in Value-based Reinforcement Learning Agents}

\author{\IEEEauthorblockN{1\textsuperscript{st} Jonathan Keane}
\IEEEauthorblockA{\textit{Department of Computer Science and Software Engineering} \\
\textit{Milwaukee School of Engineering}\\
Milwaukee, WI USA \\
keanej@msoe.edu}
\and
\IEEEauthorblockN{2\textsuperscript{nd} Sam Keyser}
\IEEEauthorblockA{\textit{Department of Computer Science and Software Engineering} \\
\textit{Milwaukee School of Engineering}\\
Milwaukee, WI USA \\
keysers@msoe.edu}
\and
\IEEEauthorblockN{3\textsuperscript{rd} Jeremy Kedziora}
\IEEEauthorblockA{\textit{Department of Computer Science and Software Engineering} \\
\textit{Milwaukee School of Engineering}\\
Milwaukee, WI USA \\
kedziora@msoe.edu}
}

\maketitle

\begin{abstract}
\noindent The use of reward functions to structure AI learning and decision making is core to the current reinforcement learning paradigm; however, without careful design of reward functions, agents can learn to solve problems in ways that may be considered ``undesirable" or ``unethical.  Without thorough understanding of the incentives a reward function creates, it can be difficult to impose principled yet general control mechanisms over its behavior.  In this paper, we study methods for constructing guardrails for AI agents that use reward functions to learn decision making.  We introduce a novel approach, which we call strategy masking, to explicitly learn and then suppress undesirable AI agent behavior.  We apply our method to study lying in AI agents and show that it can be used to effectively modify agent behavior by suppressing lying post-training without compromising agent ability to perform effectively.

\end{abstract}


\section{Introduction}
\noindent The use of reward functions to structure AI learning and decision making is core to the current reinforcement learning paradigm (\cite{suttonbarto2018}) and widely applied to fine-tune large language models.  Without careful design of reward functions, agents can learn to solve problems in ways that may be considered ``undesirable" or ``unethical."  Without a thorough understanding of the incentives a reward function creates, it can be difficult to impose principled yet general control mechanisms over its behavior.\footnote{For instance, if an AI agent takes an action inconsistent with correct information this may be human-interpreted as a hallucination.  However, it may not be clear whether the agent made an `honest mistake' or learned to act 'dishonestly' and that dishonest behavior is in fact optimal from the perspective of the agent.  It also may not be clear how to intervene to prevent it from happening again.}

In this paper, we study methods for constructing guardrails for AI agents that use reward functions to learn decision making.  We introduce a novel approach, which we call strategy masking, to explicitly learn and then suppress undesirable AI agent behavior.  Our key insight is that principled control over an AI agent must come from understanding and adjusting the incentives that lead it to make decisions.  Accordingly, we leverage techniques from the literature on expedited learning and explainability to decompose an agent's reward function (and thereby its long-term state-action values) into multiple dimensions during training.  Specific dimensions encode the expected value of employing specific agent behaviors. We then mask these dimensions to provide the AI agent access to only a subset of them for use in decision making during inference, enabling a user to adjust the learned incentives of the agent and thereby guide it to desirable behavior. We refer to this approach as strategy masking; it can be used with any reward function and value-based model architecture during or after training.

Decomposition of rewards has previously been used in hierarchical reinforcement learning (\cite{karlsson1994, Dietterich2000,shu2018,tham1994,Seijen2017}), as a means for expediting learning (\cite{russell2003, icarte2018}), or for explainability of agent decisions (\cite{juozapaitis2019}).  Perhaps closest in form to our work is \cite{badia2020}, which is aimed at optimizing agent exploration by learning Q-values based on balancing additive intrinsic and extrinsic rewards.  Our work differs from this in that we seek to use reward decomposition and masking as a means for shaping agent behavior by altering reward dimensions associated with specific action-level characteristics. While prior works have explored learning action masks on a per-state basis (\cite{wang2024}) to adjust agent behavior, we are not aware any such approach employed on reward action-level characteristics.

We apply our method to study lying in AI agents. We focus on lying because it is informationally rich and it is difficult to detect after the fact by its very nature.  Beyond that, information integrity (e.g. large language model hallucination and generation of deepfake content) is a key AI safety concern. As such, the risk of learned dishonest behavior is also a significant barrier to AI adoption in industries where consistency between verified information and AI generated content are of value.  Standard benchmark reinforcement learning environments (\cite{brockman2016}) are not well-suited for studying lying, so we use the popular social deception game Coup as a testing environment for our method.  In Coup, every player begins with private information not available to opponents and can take actions that are explicitly intended to misrepresent that information, i.e. to function as lies.  Limited work has been done applying artificial intelligence in the game of Coup, but Coup shares several similarities with poker which has been studied in greater depth (\cite{brown2017,brown2019,shi2022}).\footnote{Coup has several advantages over poker for our purposes: it lacks betting and so is a simpler environment; lying is unambiguous in Coup (but not in poker) and so it provides well-defined actions for targeted incentives of information management; it is also a comparatively short game with a high density of information relative to game length.}

For our experiments, we adapt the DQN family of value-based algorithms with function approximation (\cite{mnih2013}), apply reward decomposition and strategy masking to it, and use the resulting algorithm in our multi-agent Coup setting to train agents.  We establish that AI agents trained in the Coup environment can learn to lie to various degrees depending on the reward function the agent trains with. We show that strategy masking can effectively modify agent behavior by suppressing, or actively penalizing, the reward dimension for lying such that agents act more honestly while not compromising their ability to perform effectively (i.e. by winning games).

\section{Reward Decomposition \& Strategy Masking for TD(0) Algorithms}
\noindent In reinforcement learning, control problems are modeled as a Markov Decision Process (MDP) which is defined by a set of states $S$ that describe the current environmental conditions facing the agent, a set of actions $A$ that an agent can take, the probabilities $p(s'\mid a,s$) for transitioning from state $s$ to state $s'$ given action $a$, and a function $r:S\times A\times S \to \mathbb{R}$ so that $r(s^{\prime},a,s)$ supplies the immediate reward associated with this transition. In environments that take place across a finite number of discrete periods $T$, the sequence of periods the agent participates in is referred to as an episode. The goal of the agent is to learn a policy $\pi(a|s)$, which describes the probability that a trained agent should take action $a$ in state $s$, to maximize the sequence of rewards across across an episode: $\sum_{t = 0}^T\gamma^tr(s_{t+1},a_t,s_t)$.  Here $a_t$ and $s_t$ are the action and state at time $t$ and $\gamma\in[0,1]$ is the discount factor on future rewards.

In value-based reinforcement learning, the agent learns to predict the long term expected value for state-action pairs iteratively by repeatedly estimating the $Q$-values, defined recursively for a given policy $\pi$ as:
\begin{align*}
Q_{\pi}&(s,a)\\
&= \sum_{s^{\prime}\in S}p(s^{\prime}|a,s)\left[r(s^{\prime},a,s) + \gamma\sum_{a\in A}\pi(a|s^{\prime})Q_{\pi}(s^{\prime},a)\right].
 \end{align*}
The goal is then to update the policy to put high probability on actions with large $Q$-values, for example by using an $\varepsilon$-greedy approach where:
\begin{align*}
    \pi(a|s) = \left\{\begin{array}{ll}
    1-\varepsilon + \frac{\varepsilon}{|A|}&\mbox{if }a\in\argmax_{a\in A}\{Q_{\pi}(s,a)\}\\
    \frac{\varepsilon}{|A|}&\mbox{otherwise.}
    \end{array}\right.
\end{align*} 

\subsection{Reward Decomposition \& Strategy Masking}
\noindent In general, there may be multiple factors that contribute to aggregate rewards, to $Q_{\pi}(\cdot)$, and so to the decisions that agents make.  By modeling rewards and the long-term state-action values as scalars, traditional value-based approaches (e.g. $Q$-learning, DQN, etc.) obscure the contribution that each factor makes to agent value estimates and prevents users from adjusting those contributions to guide the agent to appropriate behavior.  To address this, we apply reward decomposition (\cite{juozapaitis2019}) and strategy masking; we define a vector-valued function where each dimension represents a single factor contributing to the overall immediate reward:
\begin{align*}
\vec{r}(s^{\prime},a,s) = \langle r_k(s^{\prime},a,s)\rangle_{k=1,\hdots, K}\in \mathbb{R}^K.
\end{align*}
Given this, it is natural to also define a vector-valued state-action value function:
\begin{align*}
\vec{Q}_{\pi}(s,a) = \langle Q_{\pi}^{(k)}(s,a)\rangle_{k=1,\hdots, K}\in \mathbb{R}^K,
\end{align*} 
where dimension $k$ reflects the contribution to the long-term state-action value from the factor measured by dimension $k$ of $\vec{r}$ so that:
\begin{align*}
Q_{\pi}^{(k)}(s,a) = E\left(\left.\sum_{t=0}^T\gamma^tr_k(s_{t+1},a_t,s_t)\right|s,a,\pi\right).
\end{align*}
The sum across these dimensions can be expressed as a dot product between the decomposed state-action value vector $\vec{{Q}}_{\pi}(\cdot)$ and a vector of coefficients whose role is to select dimensions of the decomposed state-action value function for inclusion into the sum and weight them as appropriate.  We call this coefficient vector the strategy mask and write it as:
\begin{align*}
\vec{m} = \langle m_k\rangle_{k=1,\hdots,K}.
\end{align*}
In standard reward decomposition, $\vec{{m}}$ would be a vector of ones.  Setting the $k$th entry of $\vec{m}$ to 0 would mean suppressing the contribution that the factor measured by $r_k(\cdot)$ makes.  Setting it to a negative value would amount to punishing the agent for taking an action possessed of that factor.

It is straightforward to adapt standard update rules to incorporate reward decomposition and strategy masking by working with the decomposed state-action values as appropriate.  Here we give three TD(0) examples: masked SARSA, masked expected SARSA, and masked $Q$-learning.  As a preliminary, given $\vec{Q}$ and $\vec{m}$ define a masked $\varepsilon$-greedy policy as:
\begin{align*}
\pi_{\vec{m}}(a|s) = \left\{\begin{array}{ll}
    1-\varepsilon + \frac{\varepsilon}{|A|}&\mbox{if }a = a^*_{\vec{m}}(s)\\
    \frac{\varepsilon}{|A|}&\mbox{otherwise}
    \end{array}\right.
\end{align*}
where $a^*_{\vec{m}}(s) = \argmax_{a\in A}\{\vec{Q}(s,a)\cdot\vec{m}\}$.
With this in mind we may specify SARSA with strategy masking as:
\begin{align*}
\vec{Q}(s,a)\gets(1 &- \alpha)\vec{Q}(s,a)+\alpha\left[\vec{{r}}(s^{\prime},a,s) + \gamma\vec{{Q}}\left(s',a^{\prime}\right)\right]\label{eq:Decomposed-DQN}\tag{masked SARSA}
\end{align*}
where $a\sim\pi_{\vec{m}}(s)$ and $a^{\prime}\sim\pi_{\vec{m}}(s^{\prime})$.  Similarly, a version of expected SARSA with strategy masking can be written as:
\begin{align*}
\vec{Q}(s,a)\gets(1 &- \alpha)\vec{Q}(s,a)\nonumber\\
&+\alpha\left[\vec{{r}}(s^{\prime},a,s) + \gamma\sum_{a^{\prime}}\pi_{\vec{m}}(a^{\prime}|s^{\prime})\vec{{Q}}\left(s',a^{\prime}\right)\right].\label{eq:Decomposed-DQN}\tag{masked expected SARSA}
\end{align*}
Finally, a natural candidate for $Q$-learning with strategy masking is:
\begin{align*}
\vec{Q}(s,a)\gets(1 - \alpha)&\vec{Q}(s,a)+\alpha\left[\vec{{r}}(s^{\prime},a,s) + \gamma\vec{{Q}}\left(s',a^*_{\vec{m}}(s^{\prime})\right)\right].\label{eq:Decomposed-DQN}\tag{masked $Q$-learning}
\end{align*}
Here, we build agent-understanding of the constraints on its expected future behavior into the estimation process for $\vec{Q}$ by basing the update on the next-period action associated with the state-action values included by the strategy mask. In SARSA we sample actions from the masked $\varepsilon$-greedy policy.  In expected SARSA we weight the next-period state-action values by the masked $\varepsilon$-greedy policy.  And in $Q$-learning we condition the update on a next-period action that maximizes the masked $Q$-values.

Incorporating the strategy mask $\vec{m}$ into temporal difference updating is an opportunity for the agent to learn to make good decisions while incorporating user-provided constraints on expected future behavior.  It is also an opportunity to adjust agent behavior after training by tuning the mask (note that when the training mask differs from the mask used during inference, this approach is distinct from simply imposing a reward decomposition structure with fewer reward dimensions).  


\subsection{What About Function Approximation?}
\noindent For very large state-spaces, it is common to use function approximation to estimate values for state-action pairs that the agent may never have actually encountered during training.  It is also straightforward to apply these ideas to value-based reinforcement learning approaches that incorporate function approximation.  We demonstrate this by extending the popular DQN algorithm (\cite{mnih2013}).  In DQN $Q_{\pi}(\cdot|\mathbf{w})$ is approximated as a deep learning model with parameters $\mathbf{w}$.  This deep learning model is trained by using a second deep learning model $Q_{\pi}(\cdot|\mathbf{w}^{\prime})$, referred to as the target approximator, to construct a target value that the main deep learning model can be adjusted towards.  The parameters of the target approximator, $\mathbf{w}^{\prime}$, are a lagged copy of $\mathbf{w}$ updated infrequently to maintain stability.  To incorporate reward decomposition and strategy masking into DQN we specify the target as:
\begin{align}
\vec{y} = \vec{r}(s^{\prime},a,s) + \left\{\begin{array}{ll}
   \gamma\vec{Q}(s',a^*_{\vec{m}}(s^{\prime})\mid \mathbf{w}')  & \spr\mbox{ not terminal} \\
  0   & \mbox{otherwise}
\end{array}\right.,\label{eq:DQN}
\end{align}
where $a^*_{\vec{m}}(s) = \argmax_{a\in A}\{\vec{Q}(s,a|\mathbf{w}^{\prime})\cdot\vec{m}\}$.  We include psuedocode for masked DQN below.

\begin{algorithm}[H]
\caption{Masked DQN}\label{algo:mdqn}
\begin{algorithmic}[1]
\Require{$\varepsilon,\gamma,\alpha\in(0,1)$, and $B,N,M,C\in\mathbb{N}$}
\State \textbf{Init:} $\mathbf{w}\in\mathbb{R}^d$, replay buffer $D$ with fixed memory $M$
\State Set $\mathbf{w}^{\prime}=\mathbf{w}$
\For{episode $1,2,3,\hdots,N$}
\State Sample $s\in S$
\For{$t = 1,2,3,\hdots$}
\State Sample $a$ according to $\pi_{\vec{m}}(s)$, observe $\vec{r}$ and $\spr$
\State Add $(s,a,\vec{r},\spr)$ to $D$, if $|D|>M$ drop the oldest
\If{$|D|\geq B$}
\State Sample $\left(s^{(j)},a^{(j)},\vec{r}^{(j)},s^{\prime(j)}\right)_{j=1}^B$ from $D$
\State Compute $\vec{y}^{(j)}$ according to equation \ref{eq:DQN}
\State Update: 
\begin{align*}
    \mathbf{w}\leftarrow \mathbf{w} - \alpha\nabla_{\mathbf{w}}\left(\frac{1}{B}\sum_{j=1}^B\sum_{k=1}^K\left[\vec{Q}^{(k)}\left(s^{(j)},a^{(j)}|\mathbf{w}\right)-\vec{y}^{(j)}_k\right]^2\right)
\end{align*}
\State Every $C$ steps $\mathbf{w}^{\prime}\gets \mathbf{w}$
\EndIf
\State Update $s\gets\spr$
\EndFor
\EndFor
\end{algorithmic}
\end{algorithm}

\subsection{Convergence of Masked $Q$-learning}
\noindent It is straightforward to show that reward decomposition and strategy masking can be combined with an appropriate update rule in a way that guarantees that estimates of the state-action values converge to their optimal values in the long run.  For example, for $Q$-learning, we show the following result:
\begin{theorem}Consider an MDP and suppose that $S$ and $A$ are finite, and that $\vec{r}$ and $\vec{m}$ are bounded. The quantity $\vec{Q}(s,a)\cdot\vec{m}$ converges to the optimal $Q$-value under a restricted masked $Q$-learning update rule of the form:
\begin{align*}
\vec{Q}(s,a)\cdot\vec{m} = (1-\alpha)&\vec{Q}(s,a)\cdot\vec{m}\\
&+ \alpha\left[\vec{r}\cdot\vec{m} + \gamma\vec{Q}(\spr,a^*_{\vec{m}}(\spr))\cdot\vec{m}\right]
\end{align*}
with probability one so long as:
\begin{align*}
\sum_t\alpha_t(s,a) = \infty\mbox{ and }\sum_{t}\alpha_t^2(s,a) < \infty
\end{align*}
for all $s\in S$ and $a\in A$.\label{thm:masked_Q}
\end{theorem}
\noindent The argument for Theorem \ref{thm:masked_Q} follows a typical approach by showing that the update rule is a contraction mapping and then citing the appropriate theorem (in this case Theorem 1 from \cite{Jaakkola1993}).  The full proof is in the appendix.

Note that this result shows that the decomposed estimated $Q$-values aggregated according to the strategy mask will converge to the optimal $Q$-values.  It does not offer conclusions about dimensions with a mask weight of $0$ that we include to generate estimates to be potentially used during inference to adjust agent behavior.

\section{An Environment to Study Lying: Coup}
\noindent In the remainder of the paper we apply reward decomposition and strategy masking to study lying in AI agents.  To do so, we use the popular social deception game Coup as an environment to train and test our methods, which we describe here.

\subsection{Game Structure}
\noindent Coup is a multiplayer card game.  Each player has two resources available to them: cards, described in Table \ref{tab:Coup-Roles}, which determine the actions they can legally take; coins, which they may pay to take certain actions. At the start of the game, players are given 2 coins as well as 2 cards from the deck. Players do not know the cards of the other players. The goal of the game is to eliminate opponents by forcing them to discard all their cards and be the last player standing.

At each turn, a player is allowed to choose \textbf{any} action from Table \ref{tab:Coup-Roles}.  Following this, any player may attempt to block that action if it can be blocked.  Blocked actions are not carried out.  Whenever there is an attempted action or an attempted block, any player may choose to challenge the action or block, meaning that the player who initiated the action or block must reveal that they possess a card that allows them to perform the action or block legally.  If the challenged player does not have the required card, then they must discard; otherwise the challenger must discard (with the challenged player adding their revealed card to the deck and drawing a new card).  Once a player has discarded all their cards, they have lost.\footnote{Discarded cards are not added back to the deck but are instead shown for the remainder of the game.}

\begin{table}
\caption{The actions in Coup and the cards they correspond with.  Each card describes a role that delineates the actions a player may legally take.  There are 5 different roles and 3 copies of each.  Some actions are available to any role.\break}
\label{tab:Coup-Roles}
  \centering
  \begin{tabular}{lll}
    \toprule
    \textbf{Card/Role} & \textbf{Action} \\
    \midrule
    \textbf{Duke} & \textbf{Tax}: Take 3 coins from the center of the table.\\ & \textbf{Block:} Foreign Aid.& \tabularnewline
    \midrule
    \textbf{Ambassador} & \textbf{Exchange: }Take 2 cards from the deck and \\ & choose 2 to keep.\\ & \textbf{Block:} Steal. & \tabularnewline
    \midrule
    \textbf{Captain} & \textbf{Steal: }Take 2 coins from an opposing player.\\ & \textbf{Block:} Steal. & \tabularnewline
    \midrule 
    \textbf{Contessa} &\textbf{Block:} Assassinate. & \tabularnewline
    \midrule 
    \textbf{Assassin} & \textbf{Assassinate: }Pay 3 coins to attempt to make \\ & opposing player lose a card.\tabularnewline
    \midrule 
    \textbf{Any} & \textbf{Income: }Take one coin from the center of the table.\\ & Cannot be blocked.\tabularnewline
    \midrule 
    \textbf{Any} & \textbf{Foreign Aid: }Take two coins from the center\\ & of table (can be blocked \\ & by player claiming Duke).\tabularnewline
    \midrule 
    \textbf{Any} & \textbf{Coup: }Pay 7 coins to make opposing player lose\\ & a card. Cannot be blocked. & \tabularnewline
    \bottomrule
  \end{tabular}
\end{table}

\subsection{Information and Lying}
\noindent Coup's information environment facilitates studying whether our reward decomposition/strategy masking approach can be successfully applied to manage AI truthfulness.  In Coup, every player possesses private information and is not required to take legal actions that are consistent with that private information.  They are only punished if they are caught taking illegal actions (by a challenge).  Thus, in every turn, every player has multiple opportunities to misrepresent their private information to other players by explicitly claiming to have cards that they do not actually hold in their hand.  In other words, they can lie. 



\section{Learning and Suppressing Lying}
\noindent Given the structure of the game environment discussed above, we conceptualize Coup as a partially observable multi-agent MDP and detail our approach for training an agent to play it below.

\subsection{Partial Observability}
\label{section:arch}
\noindent We deal with the incomplete information available to players by training agents to play it using an adapted version of DQN (\cite{mnih2013}), namely the DRQN approach of \cite{hausknecht2015}.  Under this approach the function approximator that models $\vec{Q}(s,a|\mathbf{w})$ incorporates sequence modeling layers, in this case an LSTM.  The agent model takes as an input the player's personal information, i.e. the cards in their hand and the number of coins they have, known information about other players, and the history of actions taken by all players in the game. We expand on the deep learning architecture we use to model $\vec{Q}(s,a|\mathbf{w})$ in Appendix \ref{Appendix - Defining an Agent for Coup}.

\subsection{League play for Multi-Agent Capability}
\label{section:league-play}
\noindent To accommodate the multi-agent structure, we gathered data to train our agent using a simplified version of the league play developed in \cite{vinyals2019}, nicknamed StarLite. In our league play, the agents being trained, also referred to as champions, sample opponents from a league of previous players based on their win-rate against these agents, an approach \cite{vinyals2019} refer to as prioritized fictitious self-play (PFSP). Champions play a combination of league players sampled using PFSP as well as current champions being sampled at a fixed rate. They also face agents trained to exploit specific strategies. The goal of this process is to prevent our champion agents from overfitting to specific opponent policies during training.  We expand upon the StarLite implementation in Appendix \ref{Appendix - Creating Opponents to Train Against}.

\subsection{Applying Reward Decomposition \& Strategy Masking to Coup}
\noindent We settled on a four-dimensional reward decomposition aimed at balancing model complexity against capturing a minimal representation of how agents could manage the information environment in Coup\footnote{In our experience, the more granular the reward decomposition, the larger the neural network needed to be to make sure there was enough model capacity to generate outputs for all reward dimensions.  }.  These dimensions were winning, challenging, lying, and baiting:  

\begin{align}
\vec{r} = \langle r_k \rangle, \mbox{where } k \in\{\mbox{Win},\mbox{Challenge},\mbox{Lie},\mbox{Bait}\}.
\end{align}

\noindent Winning and lying are self-explanatory.  The challenging dimension is meant to capture the value of successfully detecting lies by opponents.  The baiting dimension is meant to capture the value of tricking opponents into unsuccessful challenges.  See Table \ref{tab:Reward-dimension-descriptions} for further description.  The ``Win" dimension is defined with a higher value compared to the other dimensions because it is rewarded sparsely yet remains the key goal for the agent in playing the game.  For our experiments, a reward of 10 was given for the ``Win" dimension and a reward of 1 was given for all other dimensions.

\begin{table}
\caption{Reward dimension descriptions.
\break}
\label{tab:Reward-dimension-descriptions}
  \centering
  \begin{tabular}{cl}
    \toprule 
    \textbf{Reward Dimension} & \textbf{Description}\tabularnewline
    \midrule
    \textbf{Win} & Reward given when agent is the last player\\ & remaining. No punishment for losing.\tabularnewline
    \midrule 
    \textbf{Challenge} & Reward given when agent successfully\\ & challenges, forcing a discard from an opponent.\\ & No punishment for  a failed challenge.\tabularnewline
    \midrule 
    \textbf{Lie} & Reward given when agent successfully completes\\ & an action/block while not having the card for the\\ & role. No punishment for failed lie.\tabularnewline
    \midrule 
    \textbf{Bait} & Reward given when agent successfully gets an\\ & opponent to challenge erroneously. No reward\\ & for unchallenged honest play.\tabularnewline
    \bottomrule
  \end{tabular}
\end{table}


From the \ref{eq:Decomposed-DQN} update rule, if we want to create an agent oriented towards certain behaviors, such as lying, or lie-detecting via challenging, we can define $\vec{{m}}$ with ones/zeros in dimensions based on how/whether we want the agent to learn to incorporate each dimension in its learned policy.\footnote{The training mask value that minimally compromises the optimality of the learned policy if a different inference mask value is used during inference is an open question.  We hypothesize that a training mask value of 0 is the best candidate.  This intuition comes from comparison to leaving a dimension unmasked (weight of 1) during training. Applying a negative value during inference would be a large shift from reward to punishment, leading to qualitatively different incentives between training and inference, whereas 0 reflects a more neutral stance on a dimension.}  We refer to dimensions where $m_k=1$ as unmasked.  There are eight different strategy masks available given our reward decomposition that prioritize winning, listed in the columns of Table \ref{tab:possible-strategy-masks}.

\begin{table}[H]
\caption{Potential strategy masks with incentives towards winning.
\break}
\label{tab:possible-strategy-masks}
  \centering
  \begin{tabular}{r|cccccccc}
    \mbox{Win}&1&1&1&1&1&1&1&1\\
    \mbox{Challenge}&0&1&0&0&1&1&0&1\\
    \mbox{Lie}&0&0&1&0&0&1&1&1\\
    \mbox{Bait}&0&0&0&1&1&0&1&1
  \end{tabular}
\end{table}

\section{Results}

\subsection{Training Agents to Lie and Lie Detect\label{liar-detector-masks-section}}

\begin{figure}
  \includegraphics[width=1.\linewidth]{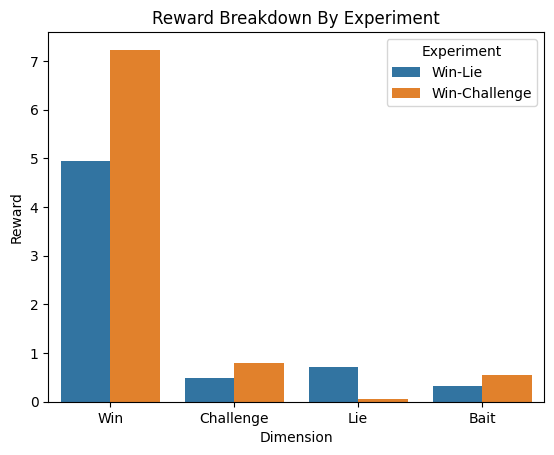}
  \caption{Breakdown of average reward per dimension over all state/action pairs for Win-Lie and Win-Challenge agents}
  \label{fig:reward-breakdown}
\end{figure}

\noindent To illustrate the efficacy of reward decomposition for creating and explaining different behaviors, we trained two agents with the StarLite league play structure. The first agent masked out all dimensions other than Win and Lie (the mask from column 3 of Table \ref{tab:possible-strategy-masks}, referred to as a Win-Lie agent); the other masked out all dimensions except Win and Challenge (the mask from column 2 of Table \ref{tab:possible-strategy-masks}, referred to as a Win-Challenge agent). Both agents were trained otherwise identically for one million episodes each. The intention was to create both a strong Win-Lie agent and a strong Win-Challenge agent. 

For each agent we then played 5000 games against their own league directly after training, holding model weights and PFSP values constant. We visualize the average reward values per dimension realized during play for either agent in Figure \ref{fig:reward-breakdown}.  For either agent their top two dimensions by reward were the dimensions we biased them towards, demonstrating that we were able to leverage reward decomposition and strategy masking to create agents that prioritized different behaviors in the same environment in a predictable, controllable way.

The distribution of the remaining dimensions reflect the leagues each agent was trained in. In the Win-Challenge league agents will aggressively challenge actions, increasing the value of the Bait dimension. As a consequence, lying is punished even though we do not manipulate the strategy mask to negatively weight lying. The Win-Lie agent learns that there is some value in baiting and challenging because the Win-Lie league has agents who will frequently attempt to lie on their turn, which means that challenges and baits will be rewarded. However, because these dimensions were masked during training, the agent did not learn to prioritize them over lying.

\begin{table}
  \caption{Win-Lie and Win-Challenge strategy mask win percentage and distribution
of rewards across reward dimensions for the three different sub-objectives
during the last 5000 episodes of training for a three-player game
of Coup. \break }
\label{tab:Liar-and-lie-comparison}
    \tabularnewline
  \centering
  \begin{tabular}{cccccc}
 \toprule 
 & & \multicolumn{4}{c}{\textbf{Collected Reward \%}}                   \\
 \cmidrule(r){3-6}
\textbf{Agent} & \textbf{Win\%} & Win & Lie & Challenge & Bait
\tabularnewline
\midrule 
Win-Lie & 30.56 \% & 78.22 \% & 7.41 \% & 9.09 \% & 5.28 \%\tabularnewline
\midrule
Win-Challenge & 42.58 \% & 83.72 \% & 9.10 \% & 0.79 \% & 6.39 \%\tabularnewline
\bottomrule
  \end{tabular}
\end{table}

\subsection{Altering Agent Behavior after Training}
\noindent Strategy masking also permits changing the mask after training and so offers an opportunity to minimize undesirable behavior that emerged during training. To demonstrate this, we trained an agent in the same league-based structure that was used to train the agents in Section \ref{liar-detector-masks-section}, only this time using the solely the win-dimension during training (the mask from column 1 of Table \ref{tab:possible-strategy-masks}).  This agent has built into its estimates of future value the expectation that its future actions will prioritize only winning, rather than specific behavioral patterns.

To analyze the efficacy of strategy masking to minimize lying post-training, we took a counterfactual approach.  Specifically, we recorded the trained agent’s actions over 5000 games against the training league, holding model weights and PFSP values constant. We then compare these historical, realized actions against the actions that would have been taken if lying was unmasked (Figure \ref{fig:action-distribution_0}) and if lying had been punished with a lie dimension weight of $-1$ (Figure \ref{fig:action-distribution_-1}), dividing actions up by type and whether a lie was involved (unmasked actions are on the left).\footnote{Using a greedy policy based on the learned $Q$-values.} In both cases, altering the mask after training substantially impacts agent tendency to lie during inference.

\begin{figure}
  \includegraphics[width=1.\linewidth]{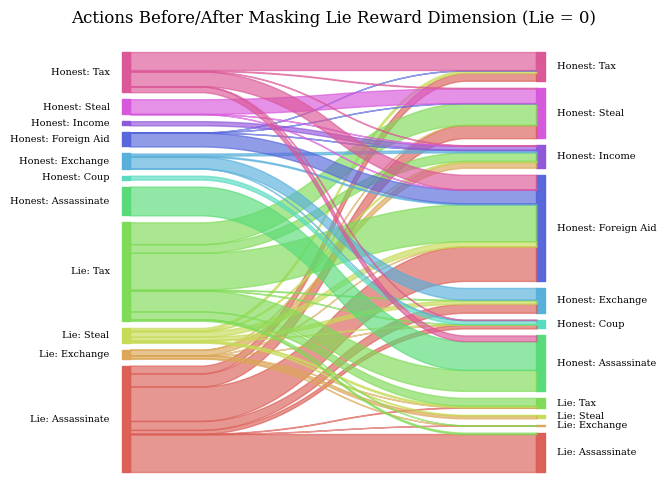}
  \caption{Comparison of distribution of actions that would have been taken with lie dimension unmasked (left) and lie dimension masked with a weight of 0 (right) across 5000 games.}
  \label{fig:action-distribution_0}

  \includegraphics[width=1.\linewidth]{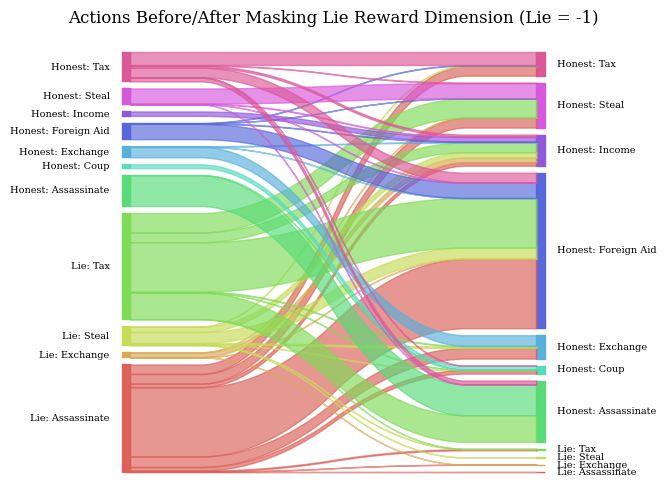}
  \caption{Comparison of distribution of actions that would have been taken with lie dimension unmasked (left) and lie dimension masked with a weight of -1 (right) across 5000 games.}
  \label{fig:action-distribution_-1}
\end{figure}

\subsection{Maintaining Agent Performance}
\noindent In our final set of experiments, we show that it is possible to alter agent behavior as above without negatively impacting agent performance.  The foregoing shows that strategy masking can be used to influence the agent post-training by eliminating problematic behavior.  However, because the strategy mask is applied post-training, the agent cannot adjust its policy in response.  Applying strategy masking in this way would be of limited value if it also had the effect of diminishing agent performance against its primary task.

To evaluate impact on agent performance, we took the agent trained with the mask from column 1 of Table \ref{tab:possible-strategy-masks} and varied the post-training mask weight of the lie dimension over $[-5, 5]$. For each lie dimension weight, we simulated 5000 games against the league and plot the resulting relationship between winning and lying for the agent with the blue and orange lines in Figure \ref{fig:win-lie-weights}.\footnote{Taking the priorities for PFSP that training had ended at as fixed for the purposes of this experiment.}

We see that for mask weights on the lying dimension greater than 0, lying increased.  Notably, when the lying dimension is unmasked (weight of 1) post-training, the agent lies up to 70\% of the time, with further increases as the weight goes up (orange line).  At the same time, there is a steady drop in win percent (blue line). On the other hand, when we disincentivized lying through the strategy mask (with weights less than 0), we see that lying actions decreased to near zero while the win rate was minimally impacted.

We further tested the robustness of agent performance by updating the opponent selection priorities.  Specifically, for each new mask, each agent played $19500$ games, selecting opponents uniformly so that each agent could develop its own set of “challenging agents” to play against when using PFSP.\footnote{The number of games played to update priorities is calculated as $\frac{wp}{n-1}$, where $w$ is the size of the window of previous games used in calculating the priority, $p$ is the number of players in the league, and $n$ is the number of players in each game. In this experiment, $w=1000$, $p=39$, and $n=3$.} Reevaluating this experiment with updated priorities for each lie dimension weight resulted in the green and red lines in Figure \ref{fig:win-lie-weights}. While these harder priorities did lead to an overall drop in win percentage similar effects obtain. 

What we take from these findings is that changing the strategy mask weight on the lying dimension introduces two qualitatively different constraints on agent behavior, depending on whether the weight is positive or negative. With negative weights, the agent is incentivized to win without lying. With positive weights the agent is incentivized to win by lying as much as possible.  We hypothesize that in our specific environment there are a limited number of situations where lying is advantageous, so increasing the weight on the strategy mask lying dimension pushes agents to lie in increasingly risky scenarios.  On the other hand, if we disincentivize lying, the agent now only has to change its actions for the instances that it already was considering to be advantageous for lying.  Thus, negative mask weights introduce minimal constraints on its choice-space while positive weights introduce substantially larger constraints and result in a less optimal policy post-training.

\begin{figure}
  \includegraphics[width=1.\linewidth]{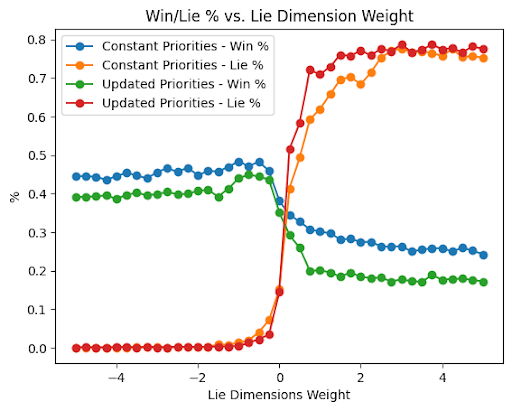}
  \caption{Across 5000 games, win percent and percentage of actions that were lies while varying the lie dimension in the strategy mask.}
  \label{fig:win-lie-weights}
\end{figure}

\section{Limitations and Future Work}
\noindent We believe that our results above show that we were able to use strategy masking to effectively control and mitigate specific agent behaviors.  There are some caveats to this that are worth mentioning and we do so now.  First, strategy masking as developed here is applicable to value-based reinforcement learning algorithms.  It is not clear how to extend strategy masking to policy-based algorithms and additional work would be needed to make this approach applicable there.  Second, our experiments in this paper focused on demonstrating the impact of strategy masking on agent behaviors under relatively fixed conditions.  Specifically, the behaviors we sought to modify were unambiguous and therefore easy to identify in our environment.  Beyond that we simply specified the reward values \textit{a priori}.  A natural next step for this work would be to test strategy masking to manage hallucination in large language models, an environment which requires estimating both the presence of hallucination and the reward structure.  More generally, we also think it would be valuable to test strategy masking in the context of inverse reinforcement learning, where the decomposed rewards could be learned.  Additionally, our experiments focused on masking out a behavior during training with a weight of 0 and then altering that during inference.  We believe that there are opportunities to experiment with the alignment between training and inference masks to measure the optimality of the resulting policy.  Finally, as this technique can be applied in any context involving a reward function, further work should be aimed towards applying this technique in different/more complex environments with larger capacity models to show that this approach can work effectively at scale in a variety of situations. For example, it would be useful to explore if reward decomposition and strategy masking can be used to track unfavorable model behaviors such as aggression in large language models trained via RLHF and systematically remove them.




\section{Conclusion}
\noindent In this work, we introduced strategy masking, a method to explicitly learn and then suppress undesirable AI agent behavior.  We demonstrated that the combination of reward decomposition and strategy masking can be used to train agents to employ specific behavioral strategies and then can be used to incentivize those same agents to change their behavior without additional training, and without compromising performance. The risk that AI agents learn bad behaviors unintentionally because of their reward structures serves as a warning about how artificial intelligence can go wrong if not properly monitored and controlled.  In our view, strategy masking could be particularly powerful for larger systems, as it can be applied without having to use extra compute to train the system further to avoid specific behaviors.  As artificial intelligence continues to be applied to more domains and problems, the risk that agents learn undesirable behaviors during training will grow.  Our hope is that this work can serve as a step towards general approaches to construct AI guardrails.


\bibliography{strategy_masking.bbl}
\bibliographystyle{icml2024}

\clearpage
\appendices
\section{Defining an Agent for Coup}
\label{Appendix - Defining an Agent for Coup}

\subsection{Agent Architecture}
\noindent Due to the partial observability of the state regarding the cards that other players have, we opted for a neural network architecture that would allow for sufficient comprehension of the actions taken in previous periods. As we wanted to encode this complex action history in the agent’s state, a function-approximator-based approach such as DQN (\cite{mnih2013}) became necessary to interpolate between states. We took inspiration for our network architecture from \cite{shi2022} and \cite{hausknecht2015} to address the partial observability of the state, which make use of Long Short Term Memory (LSTM) (\cite{hochreiter1997}) blocks to retain information from previous states.

The state available to an agent can be broken up into three types of information]: public static information, local information, and public dynamic (“historical”) information. Public static information is information about the current board that everyone can see at the current point of the game, including information such as the number of coins, lives, and revealed cards each player has. Local information is information for a given turn that only the agent knows, such as the hidden cards of the agent. Public dynamic information represents the sequence of actions taken by all players up until the current time step in the game. This historical information is considered dynamic because it has dynamic length.

Static information is concatenated into one vector and encoded via a simple MLP. The dynamic information is processed by a series of LSTM blocks and concatenated with the encoded static information. The resultant vector is passed through another MLP to produce an approximated Q-value $Q(s,a|\mathbf{w})$. This is summarized in Figure \ref{fig:d-coup-n-appendix}.

\begin{figure}
  \includegraphics[width=1.\linewidth]{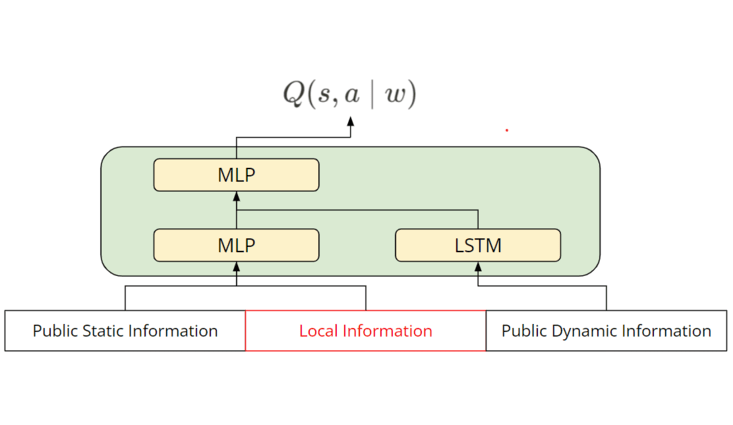}
  \caption{DQN function-approximator architecture used for training Coup agents.}
  \label{fig:d-coup-n-appendix}
\end{figure}

\subsection{Adapting for Reward Decomposition \& Strategy Masking}
\noindent To predict expected values for all actions and reward components $\vec{r}$ we altered the final layer to have $|A|\times K$ outputs. We can then calculate the values of each action as the sum of the reward components for each action in the model’s outputs. In the case of strategy masking, this would be a weighted sum using the strategy mask weights $\vec{m}$. Once these scalar values are calculated, they can be used to create the agent's policy as typically done for a DQN architecture (outputting state-action values as scalars) by using an $\varepsilon$-greedy policy.

\section{StarLite: Creating Opponents to Train Against}
\label{Appendix - Creating Opponents to Train Against}
\noindent To accommodate Coup's multi-agent structure and create capable opponents for our main agent to play against, we used a league play-based setup, inspired heavily by the work of \cite{vinyals2019}, which employed this technique for creating agents to play the complex multi-agent game of StarCraft.

In our simpler league play setup, nicknamed StarLite, we begin training agents from a random weight initialization to learn to play Coup. This is one key difference from \cite{vinyals2019}, as they had a large corpus of actions from professional human gameplay that was used to train their base agent via imitation learning. The main objective of league play is to create a set of potential agents that a learning agent can select opponents from to simulate games against and create training data to learn from.

In our league play, two main types of agents are added to the league over time to create our league: champions and main exploiters. Champions agents are checkpoints of the main learning agent that we are trying to train through league play. They play against copies of themselves and all agents in the league. To select opponents to play against during a single episode, a copy of the current learner is selected with probability $p$. Opponents are then selected from the league with probability $(1-p) \cdot f(w_{opp})$, where $w_{opp}$ represents how often the champion wins, given opponent $opp$ is in the episode, and $f(x)$ is a prioritization function that can be used to bias the learning agent to select opponents at a certain skill level. This approach is known as prioritized fictitious self-play (PFSP).

In our implementation of league play, we use $p = 0.3$ and prioritization function $f(x) = (1-x)^z$ with $z=6$ ($z$ is a parameter that prioritizes more selections of the most difficult opponents if increased). This function is used in the work from \cite{vinyals2019}. The input to $f(x)$ for $opponent_i$ is calculated as the $P(win|opponent_i) = \frac{P(win \cap opponent_i)}{P(opponent_i)}$ where $P(win \cap opponent_i)$ is calculated based on the last 1000 games involving $opponent_i$.

Main exploiters are trained from a random initialization of weights every time a new checkpoint of the champion is added to the league. These exploiters learn to explicitly counteract the strategies that the champions are developing. By adding these players to the league for future champions to learn against, they help to create a diverse set of strategies in the league, hopefully forcing future champions to create a strategy that will generalize to play well against many strategies.

Every 50 thousand episodes, a checkpoint of a champion was created and a main exploiter trained on the current state of the league for an additional 50 thousand episodes before being added to the league.

\begin{figure}
  \includegraphics[width=1.\linewidth]{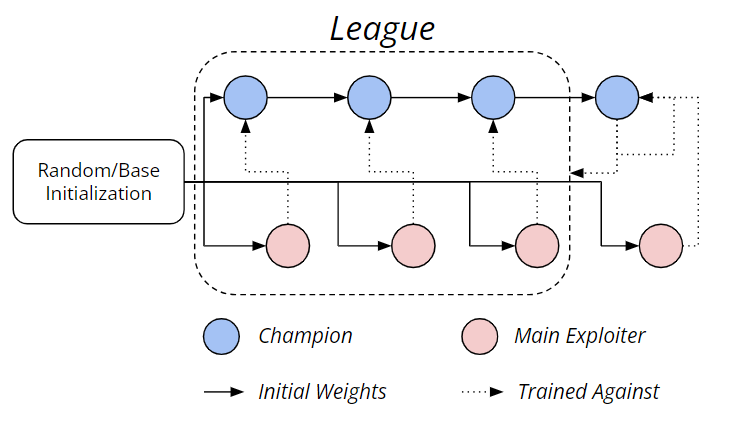}
  \caption{StarLite league play structure.}
  \label{fig:StarLite-league-play-appendix}
\end{figure}

We depict how this league play develops over time in Figure \ref{fig:StarLite-league-play-appendix}, where the left side of the figure represents $t = 0$ and progresses right during training. Circles represent checkpoints of agents and are considered part of the league if they fall within the dashed box. Solid lines represent periods of agent training, connecting the starting weights of the agent to the final weights of the agent that will be checkpointed. Dashed lines are used to describe the agents each agent trains against. A dashed line pointing to the dashed “League” box represents the use of PFSP during champion training.

\section{Asymptotic Properties of masked $Q$-learning}

\noindent Define the following operator:
\begin{align*}
H\vec{Q}(s,a)\cdot\vec{m} = \sum_{\spr\in S}p(\spr|s,a)\left[\vec{r}\cdot\vec{m} + \gamma\vec{Q}(\spr,a^*_{\vec{m}}(\spr))\cdot\vec{m}\right]
\end{align*}
where $a^*_{\vec{m}}(s) = \argmax_{a\in A}\left\{\vec{Q}(s,a)\cdot\vec{m}\right\}$.

\begin{lemma}$H\vec{Q}(s,a)\cdot\vec{m}$ is a contraction mapping in the sup-norm.\label{lemma:contraction}
\end{lemma}
\begin{proof} Take any $\vec{m}$, and any $\vec{Q}^{(1)}(s,a)$ and $\vec{Q}^{(2)}(s,a)$.  We want to show that:
\begin{align*}
||HQ^{(1)}(s,a)\cdot\vec{m}& - HQ^{(2)}(s,a)\cdot\vec{m}||_{\infty}\\
&\hspace{5mm}\leq\gamma||Q^{(1)}(s,a)\cdot\vec{m} - Q^{(2)}(s,a)\cdot\vec{m}||_{\infty}.
\end{align*}
To see this, consider:
\begin{align*}
||&HQ^{(1)}(s,a)\cdot\vec{m} - HQ^{(2)}(s,a)\cdot\vec{m}||_{\infty}\\
&=\max_{s,a}\left|\sum_{\spr\in S}p(\spr|s,a)\left[\vec{r}\cdot\vec{m} + \gamma\vec{Q}^{(1)}(\spr,a^{*,1}_{\vec{m}}(\spr))\cdot\vec{m}\right]\right.\\
&\hspace{15mm}\left.-\sum_{\spr\in S}p(\spr|s,a)\left[\vec{r}\cdot\vec{m} + \gamma\vec{Q}^{(1)}(\spr,a^{*,2}_{\vec{m}}(\spr)))\cdot\vec{m}\right]\right|\\
&=\max_{s,a}\left|\sum_{\spr\in S}p(\spr|s,a)\gamma\left[\vec{Q}^{(1)}(\spr,a^{*,1}_{\vec{m}}(\spr))\cdot\vec{m}\right.\right.\\
&\hspace{50mm}\left.\left. - \vec{Q}^{(2)}(\spr,a^{*,2}_{\vec{m}}(\spr))\cdot\vec{m}\right]\right|.
\end{align*}
where $a^{*,i}_{\vec{m}}(s) \in \argmax_a\left\{\vec{Q}^{(i)}(s,a)\cdot\vec{m}\right\}$. By the triangle inequality:
\begin{align*}
&\leq \max_{s,a}\sum_{\spr\in S}p(\spr|s,a)\gamma\left|\vec{Q}^{(1)}(\spr,a^{*,1}_{\vec{m}}(\spr))\cdot\vec{m}\right.\\
&\hspace{50mm}\left. - \vec{Q}^{(2)}(\spr,a^{*,2}_{\vec{m}}(\spr))\cdot\vec{m}\right|
\end{align*}
Note that $\vec{Q}^{(i)}(s,a^{*,i}_{\vec{m}}(s) = \max_a\left\{\vec{Q}^{(i)}(s,a)\cdot\vec{m}\right\}$.  Thus we have:
\begin{align*}
&=\max_{s,a}\sum_{\spr\in S}p(\spr|s,a)\gamma\left|\max_a\vec{Q}^{(1)}(\spr,a)\cdot\vec{m} - \max_a\vec{Q}^{(2)}(\spr,a)\cdot\vec{m}\right|\\
&\leq\max_{s,a}\sum_{\spr\in S}p(\spr|s,a)\gamma\max_a\left|\vec{Q}^{(1)}(\spr,a)\cdot\vec{m} - \vec{Q}^{(2)}(\spr,a)\cdot\vec{m}\right|\\
&\leq\max_{s,a}\sum_{\spr\in S}p(\spr|s,a)\gamma\max_{s,a}\left|\vec{Q}^{(1)}(s,a)\cdot\vec{m} - \vec{Q}^{(2)}(s,a)\cdot\vec{m}\right|\\
&=\max_{s,a}\sum_{\spr\in S}p(\spr|s,a)\gamma||\vec{Q}^{(1)}(s,a)\cdot\vec{m} - \vec{Q}^{(2)}(s,a)\cdot\vec{m}||_{\infty}\\
&=\gamma||\vec{Q}^{(1)}(s,a)\cdot\vec{m} - \vec{Q}^{(2)}(s,a)\cdot\vec{m}||_{\infty}
\end{align*}
as required.
\end{proof}

\noindent\textbf{Proof of Theorem \ref{thm:masked_Q}}.
\begin{proof}Denote the optimal $Q$-value as $Q^*_{\vec{m}}(s,a)$ and write:
\begin{align*}
\vec{Q}^{(t+1)}(s,a)\cdot\vec{m} &= \left[(1-\alpha_t(s,a))\vec{Q}^{(t)}(s,a)\right.\\
    &\left.\hspace{10mm}+ \alpha_t(s,a)\left[\vec{r} + \gamma\vec{Q}^{(t)}(\spr,a^*_{\vec{m}}(\spr))\right]\right]\cdot\vec{m}\\
%
%
&= (1-\alpha_t(s,a))\vec{Q}^{(t)}(s,a)\cdot\vec{m}\\
    &\hspace{10mm}+ \alpha_t(s,a)\left[\vec{r}\cdot\vec{m} + \gamma\vec{Q}^{(t)}(\spr,a^*_{\vec{m}}(\spr))\cdot\vec{m}\right]
\end{align*}
Now, define $\Delta_t(s,a) = \vec{Q}^{(t)}(s,a)\cdot\vec{m} - Q^*_{\vec{m}}(s,a)$.  Note further that:
\begin{align*}
\vec{Q}^{(t+1)}&(s,a)\cdot\vec{m} - Q^*_{\vec{m}}(s,a)\\
&= (1-\alpha_t(s,a))\vec{Q}^{(t)}(s,a)\cdot\vec{m} - Q^*_{\vec{m}}(s,a)\\
    &\hspace{5mm}+ \alpha_t(s,a)\left[\vec{r}\cdot\vec{m} + \gamma\vec{Q}^{(t)}(\spr,a^*_{\vec{m}}(\spr))\cdot\vec{m}\right]\\
\Leftrightarrow\Delta_{t+1}(s,a) &= (1-\alpha_t(s,a))\vec{Q}^{(t)}(s,a)\cdot\vec{m} - Q^*_{\vec{m}}(s,a)\\
    &\hspace{5mm}+ \alpha_t(s,a)\left[\vec{r}\cdot\vec{m} + \gamma\vec{Q}^{(t)}(\spr,a^*_{\vec{m}}(\spr))\cdot\vec{m}\right]\\
\Leftrightarrow\Delta_{t+1}(s,a) &= (1-\alpha_t(s,a))\vec{Q}^{(t)}(s,a)\cdot\vec{m} - Q^*_{\vec{m}}(s,a) (s,a)\\
    &\hspace{5mm} + \alpha_t(s,a)\left[\vec{r}\cdot\vec{m} + \gamma\vec{Q}^{(t)}(\spr,a^*_{\vec{m}}(\spr))\cdot\vec{m}\right]\\
    &\hspace{5mm} + \alpha_t(s,a)Q^*_{\vec{m}} - \alpha_t(s,a)Q^*_{\vec{m}}(s,a)\\
\Leftrightarrow\Delta_{t+1}(s,a) &= (1-\alpha_t(s,a))\Delta_t(s,a)\\
    &\hspace{5mm} + \alpha_t(s,a)\left[\vec{r}\cdot\vec{m} + \gamma\vec{Q}^{(t)}(\spr,a^*_{\vec{m}}(\spr))\cdot\vec{m} - Q^*_{\vec{m}}(s,a)\right].
\end{align*}
Next, write:
\begin{align*}
F_t(s,a) = \vec{r}\cdot\vec{m} + \gamma\vec{Q}^{(t)}(\spr,a^*_{\vec{m}}(\spr))\cdot\vec{m} - Q^*_{\vec{m}}(s,a)
\end{align*}
and note that $\vec{r} = \vec{r}(\spr,a,s)$.  Thus, given that $\spr\sim p(s,a)$ we have:
\begin{align*}
E(F_t&(s,a))\\
&= \sum_{\spr\in S}p(\spr|s,a)\left[\vec{r}\cdot\vec{m} + \gamma\vec{Q}^{(t)}(\spr,a^*_{\vec{m}}(\spr))\cdot\vec{m} - Q^*_{\vec{m}}(s,a)\right]\\
&= \sum_{\spr\in S}p(\spr|s,a)\left[\vec{r}\cdot\vec{m} + \gamma\vec{Q}^{(t)}(\spr,a^*_{\vec{m}}(\spr))\cdot\vec{m}\right] - Q^*_{\vec{m}}(s,a).
\end{align*}
Bythe fact that $HQ^*_{\vec{m}}(s,a) = Q^*_{\vec{m}}(s,a)$ we may write:
\begin{align*}
HQ^{(t)}(s,a)\cdot\vec{m} - Q^*_{\vec{m}}(s,a) = HQ^{(t)}(s,a)\cdot\vec{m} - HQ^*_{\vec{m}}(s,a)
\end{align*}
By Lemma \ref{lemma:contraction} we have:
\begin{align*}
||HQ^{(t)}(s,a)\cdot\vec{m}& - HQ^*_{\vec{m}}(s,a)||_{\infty}\\
&\leq\gamma||Q^{(t)}(s,a)\cdot\vec{m} - Q^*_{\vec{m}}(s,a)||_{\infty}.
\end{align*}
This verifies condition 3 of Theorem 1 in \cite{Jaakkola1993}.  Next consider:
\begin{align*}
Var(F_t(s,a) = E(F_t(s,a) - E(F_t(s,a)))^2.
\end{align*}
By the foregoing we established that:
\begin{align*}
E(F_t(s,a)) &= \sum_{\spr\in S}p(\spr|s,a)\left[\vec{r}\cdot\vec{m} + \gamma\vec{Q}^{(t)}(\spr,a^*_{\vec{m}}(\spr))\cdot\vec{m}\right] - Q^*_{\vec{m}}(s,a)\\
& = H\vec{Q}^{(t)}(s,a)\cdot\vec{m} - Q^*_{\vec{m}}(s,a).
\end{align*}
Thus we may write:
\begin{align*}
Var(F_t(s,a) &= E\left(\vec{r}\cdot\vec{m} + \gamma\vec{Q}^{(t)}(\spr,a^*_{\vec{m}}(\spr))\cdot\vec{m} - Q^*_{\vec{m}}(s,a)\right.\\
&\hspace{10mm}\left. - (H\vec{Q}^{(t)}(s,a)\cdot\vec{m} - Q^*_{\vec{m}}(s,a))\right)^2\\
&= E\left(\vec{r}\cdot\vec{m} + \gamma\vec{Q}^{(t)}(\spr,a^*_{\vec{m}}(\spr))\cdot\vec{m}  - H\vec{Q}^{(t)}(s,a)\cdot\vec{m}\right)^2\\
&= Var\left(\vec{r}\cdot\vec{m} + \gamma\vec{Q}^{(t)}(\spr,a^*_{\vec{m}}(\spr))\cdot\vec{m}\right).
\end{align*}
Since $\vec{r}$ and $\vec{m}$ are assumed bounded it follows that there exists $C\in\mathbb{R}$ such that:
\begin{align*}
Var(F_t(s,a))\leq C(1 + ||\Delta_t(s,a)||^2_{\infty}).
\end{align*}
This satisfies condition 4 of of Theorem 1 in \cite{Jaakkola1993}.  Finally, note that conditions 1 and 2 of Theorem 1 in \cite{Jaakkola1993} are satisfied by assumption.  Thus, by of Theorem 1 in \cite{Jaakkola1993} we have that
\begin{align*}
\Delta_t(s,a) = \vec{Q}^{(t)}(s,a)\cdot\vec{m} - Q^*_{\vec{m}}(s,a)\to 0
\end{align*}
with probability 1.  That is, $\vec{Q}^{(t)}(s,a)\cdot\vec{m}\to Q^*_{\vec{m}}(s,a)$ with probability 1, as required.

\end{proof}

\end{document}